\documentclass[10pt,twocolumn,letterpaper]{article}

\usepackage{cvpr}
\usepackage{times}
\usepackage{epsfig}
\usepackage{graphicx}
\usepackage{amsmath}
\usepackage{amssymb}
\usepackage{color}
\usepackage{amsthm}
\usepackage{booktabs}
\usepackage{multirow}
\usepackage{sidecap}
\newtheorem{theorem}{Proposition}

\newcommand{\bx}{\mathbf{x}}
\newcommand{\bw}{\mathbf{w}}

\newcommand{\bft}{\mathbf{f}}

\definecolor{sh_gray}{rgb}{0.84,0.84,0.84}
\definecolor{sh_gray2}{rgb}{1,0.89,0.75}
\definecolor{color3}{rgb}{0.95,0.95,0.95}
\definecolor{color4}{rgb}{0.96,0.96,0.86}

\newcommand{\best}[1]{\colorbox{sh_gray2}{\textbf{#1}}}%
\newcommand{\sbest}[1]{\colorbox{sh_gray}{\textbf{#1}}}%

\cvprfinalcopy

\def\cvprPaperID{3048} % *** Enter the CVPR Paper ID here

\ifcvprfinal\fi
\begin{document}

%%%%%%%%% TITLE
% \title{AffinityLoss: An Integrated Objective for Max-margin \\ Class Imbalanced Learning}
\title{Max-margin Class Imbalanced Learning with Gaussian Affinity}

\author{Munawar Hayat$^{1,2}$, Salman Khan$^{1,3}$, Waqas Zamir$^{1}$, Jianbing Shen$^{1,4}$, Ling Shao$^1$\\
$^1$Inception Institute of Artificial Intelligence, $^2$University of Canberra,\\ $^3$Australian National University, $^4$Beijing Institute of Technology\\
{\tt\small {firstname.lastname}@inceptioniai.org}
}

\maketitle
%\thispagestyle{empty}

%%%%%%%%% ABSTRACT
\begin{abstract}

   Real-world object classes appear in imbalanced ratios. This poses a significant challenge for classifiers which get biased towards frequent classes. We hypothesize that improving the generalization capability of a classifier should improve learning on imbalanced datasets. Here, we introduce the first hybrid loss function that jointly performs classification and clustering in a single formulation. Our approach is based on an `affinity measure' in Euclidean space that leads to the following benefits: (1) direct enforcement  of maximum margin constraints on classification boundaries, (2) a tractable way to ensure uniformly spaced and equidistant cluster centers, (3) flexibility to learn multiple class prototypes to support diversity and discriminability in feature space. Our extensive experiments demonstrate the significant performance improvements on visual classification and verification tasks on multiple imbalanced datasets. The proposed loss can easily be plugged in any deep architecture as a differentiable block and demonstrates robustness against different levels of data imbalance and corrupted labels. 
   
\end{abstract}

%%%%%%%%% BODY TEXT
\section{Introduction}
Deep neural networks are data hungry in nature and require large amounts of data for successful training. For imbalanced datasets, where several (potentially important) classes have a scarce representation, the learned models are biased towards highly abundant classes. This is because the scarce classes have less representations during training which results in a mismatch between the joint distribution model for training $p(x,y)$ and test sets $p(x',y')$. This leads to lower recall rates for rare classes, which are otherwise critically desirable in numerous scenarios. As an example, a malignant lesion is rare compared to benign ones, but should not be miss-classified. 

\begin{figure}
    \centering
    \includegraphics[width=\linewidth,keepaspectratio=True]{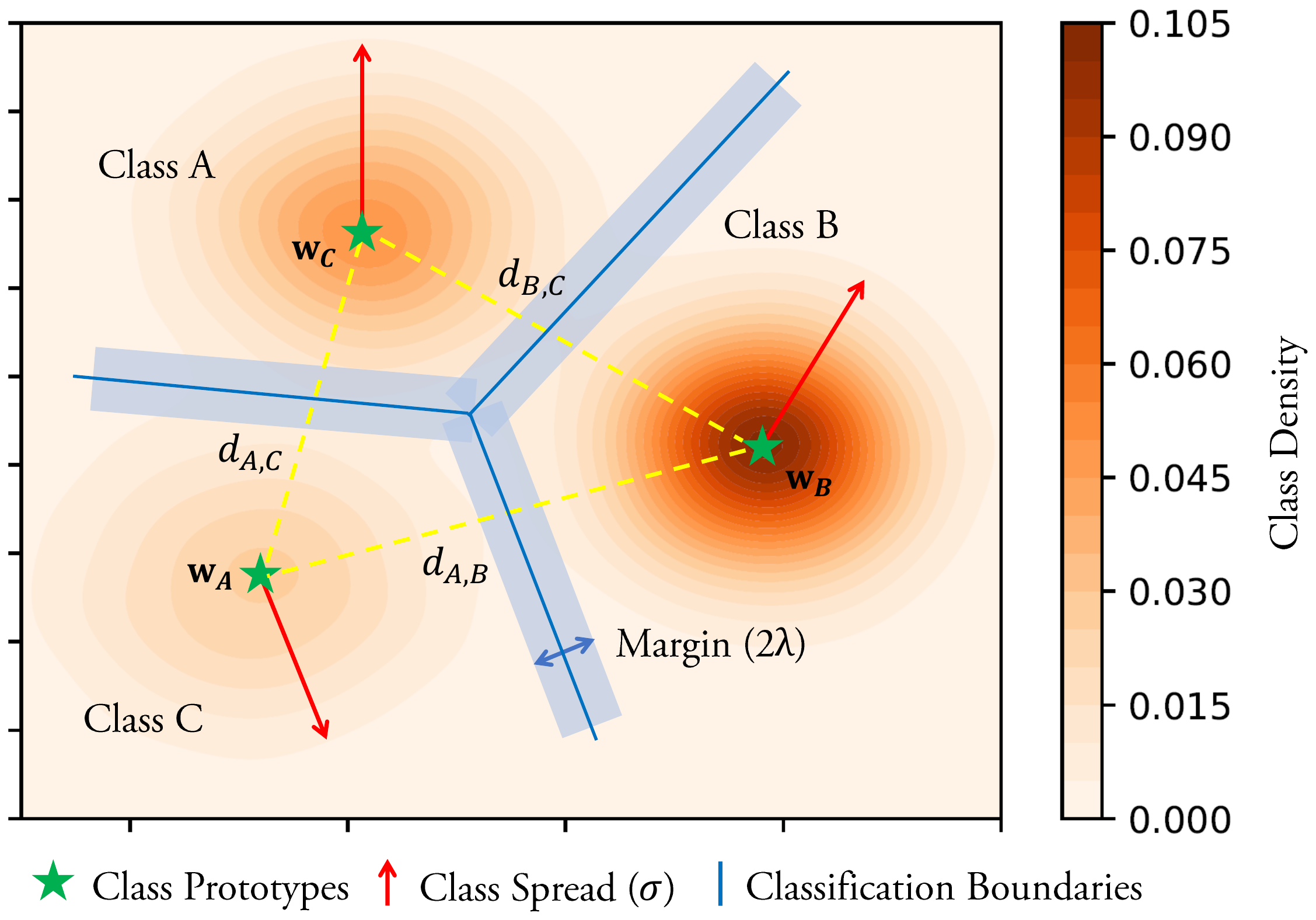}
    \caption{Affinity Loss integrates classification and clustering in a single objective. It's flexible formulation in Euclidean space allows enforcing margin between classes, control over learned clusters, number of class-prototypes and the distance between class-prototypes. Such max-margin learning greatly helps in overcoming class imbalance by learning balanced classification regions and generalizable class boundaries. }
    \label{fig:intro_fig}
\end{figure}

The soft-max loss is a popular choice for  conventional recognition tasks. However, through extensive experiments, we show that it is less suitable to handle mismatch between train and test distributions. This is partly due to no direct enforcement of margins in the classification space and the lack of a principled approach to control intra-class variations and inter-class separation. 
Here, we propose that max-margin learning can improve generalization which can help mitigate classifier bias towards more frequent classes by learning balanced representations for all classes. Remarkably, some recent efforts focus on introducing max-margin constraints within the soft-max loss function \cite{deng2018arcface, liu2016large, liu2017sphereface}. Since soft-max loss computes similarities in the angular domain (vector dot-product or cosine similarity), direct enforcement of angular margins is ill-posed and existing works either  involve approximations or make restricting assumptions (e.g., points lying on a hypersphere). % ill-posed or intractable

In this paper, we propose a novel loss formulation that enhances generalization by jointly reducing intra-class variations and maximizing inter-class distances. A notable difference from the previous works is the automatic learning of class representative prototypes in the Euclidean space with inherent flexibility to enforce certain geometric constraints on the learned prototypes. This is in contrast to soft-max loss where more abundant classes tend to occupy additional space in the projected feature space and rare classes get a skewed representation. The proposed objective is named the `Affinity loss function' as it is based on a Gaussian similarity metric defined in terms of Bergman divergence. The proposed loss formulation learns to map input images to a highly discriminative Euclidean space where the distance with class representative prototypes provides a direct similarity measure for each class. The class prototypes are key points in the  embedding space around which feature points are clustered  \cite{snell2017prototypical}. 

The affinity loss function promotes the classifier to have a simpler, balanced and more generalizable inductive bias during training. 
% because explicit inter-class margin maximization and intra-class distance minimization is promoted 
The proposed loss function thus provides the following advantages:
\begin{itemize}\setlength{\itemsep}{0pt}
\item An inherent mechanism to jointly cluster and classify feature vectors in the Euclidean space.

\item A tractable way to ensure uniformly spaced and equidistant class prototypes (when embedding dimension $d$ and prototype number $n$ are related as: $n < d+1$). 
%We can enforce constraints on the learned projection vectors to be equidistant. This is possible in the proposed formulation, as we can have $n+1$ equidistant points for an $n$-dimensional projection vector in the Euclidean space.

\item Along-with uniformly spaced prototypes, our formulation ensures that the clusters formed around the prototypes are uniformly shaped (in terms of second order moments).  %Uniform sized clusters for all classes. The parameter $\sigma$ determines the cluster size. Our formulation inherently ensures that all clusters are uniformly spaced and uniformly shaped (same size).

%\item Compared with softmax loss, where majority classes tend to occupy more space in the projected feature space, the proposed formulation ensures that the learned feature space for all classes is uniformly spaced and distributed.

\item The resulting classifier shows robustness against different levels of label noises and imbalances amongst classes.% robust to label noise and adversarial perturbations to inputs. 
\end{itemize}

%\MH{I like the introduction in this paper https://arxiv.org/pdf/1704.08063.pdf  Softmax loss and max-margin constraints in Euclid space do not make much sense, as the features are distributed in Angular space .. softmax with angular margin (s) has been the popular choice.. however, this is also not quite suitable, since the data does not necassirly has to lie on angular manifold; we can not ensure equi-distant constraints in Angular space.}

%In order to address these challenges associated with Softmax loss, individual efforts have been made in the literature. Most of the works first project both the projection projection vectors and feature vectors onto a unit sphere, where it is tractable to enforce constraints such as max-margin. 

The proposed loss function is a differentiable module which is applicable to different network architectures, and complements the commonly deployed regularization techniques including dropout, weight decay and momentum. Through extensive evaluations on a number of datasets, we demonstrate that it achieves a highly balanced and generalizable classifier, leading to significant improvements over previous techniques.

\section{Related Work}

\textbf{Class-imbalanced Learning:}
Imbalanced datasets exhibit complex characteristics and learning from such data requires designing new techniques and paradigms. The existing class imbalance approaches can be divided into two main categories, \textbf{1)} data-level, and \textbf{2)} algorithm-level approaches. The data-level schemes modify the distribution of data e.g., by oversampling the minority classes \cite{shen2016relay,chawla2002smote, han2005borderline,he2009learning, jo2004class} or undersampling the majority classes \cite{kubat1997addressing,barandela2003restricted}. Such approaches are usually susceptible to redundancy and over-fitting (for over-sampling) and critical information loss (for under-sampling). In comparison, the algorithm level approaches improve the classifier itself e.g., through cost-sensitive learning. Such methods incorporate prior knowledge about classes based upon their significance or representation in the training data \cite{lawrence2012neural,richard1991neural, khan2017cost}. These methods have been applied to different classifiers including SVMs \cite{tang2009svms}, decision trees \cite{zhou2010multi} and boosting \cite{ting2000comparative}. Some works further explore ensemble of cost-sensitive classifiers to tackle imbalance \cite{huang2018discriminative,KRAWCZYK2014554}. A major challenge associated with these cost-sensitive methods is that the class-specific costs are only defined at the beginning, and they lack mechanisms to dynamically update the costs during the course of training.%are usually , and no mechanism exists to  these costs .%  to determine the appropriate costs for different classes, specially for large-scale data \cite{huang2016learning}.

\textbf{Deep Imbalanced Learning:} Some recent attempts have been made to learn deep models from imbalanced data \cite{ jeatrakul2010classification,khan2017cost, castro2013novel,wang2016training,ng2016dual}. For example, the method in \cite{jeatrakul2010classification} first learns to under sample the training data using a neural network, followed by Synthetic Minority Oversampling TEchnique (SMOTE) based technique to re-balance the data. Deep models are trained to directly optimize the imbalanced classification accuracy in \cite{wang2016training,ng2016dual}. Wang et.~al. \cite{wang2017learning} propose a meta learning approach to progressively transfer the model parameters from majority towards less-frequent classes. Some works \cite{khan2017cost, castro2013novel} train cost sensitive deep networks, which alternatively optimize class costs and network weights. Continually determining class costs while training a deep model is still an open and challenging research problem, and makes optimization intractable in learning from large scale datasets \cite{huang2016learning}.

\textbf{Joint Loss Formulation:} Popular loss functions used for classification in deep networks include hinge loss, soft-max loss, Euclidean loss and contrastive loss \cite{khan2018guide}. A triplet loss could simultaneously perform recognition and clustering, however its training is prohibitive due to huge number of triplet combinations on large-scale datasets \cite{schroff2015facenet}. Since these loss functions are limited in their capability to achieve discriminability in feature space, recent literature explores the combination of multiple loss functions. To this end, \cite{sun2014deep} showed that the combination of soft-max and contrastive losses concurrently enforce intra-class compactness and inter-class separability. On a similar line,  \cite{wen2016discriminative} proposed `center loss' that uses separate objectives for classification and clustering. 

\textbf{Max-margin Learning:} Margin-maximizing learning objectives have been traditionally used in machine learning. Hinge loss in Support vector machines is one of the pioneering max-margin learning framework \cite{hearst1998support}. Some recent works aim to integrate max-margin learning with cross-entropy loss function. Among these, Large-margin soft-max \cite{liu2016large} enforces inter-class separability directly on the dot-product similarity while SphereFace \cite{liu2017sphereface}  and ArcFace \cite{deng2018arcface} enforce multiplicative and additive angular margins on the hypersphere manifold, respectively. The hypersphere assumption for feature space makes the resulting loss less generalizable to applications other than face recognition. Furthermore, enforcing margin based separation in angular domain is an ill-posed problem and either requires approximations or assumptions (e.g., unit sphere) \cite{elsayed2018large}. THis paper proposes a new flexible loss function which simultaneously performs clustering and classification, and enables direct enforcement of the max-margin constraints. We describe the proposed loss formulation next.
\section{Max-margin Framework}
We propose a hybrid multi-task formulation to perform learning on imbalanced datasets. The proposed formulation combines classification and clustering in a single objective that minimizes intra-class variations while simultaneously achieving maximal inter-class separation. We first explain why traditional Soft-max Loss (SL) is unsuitable for large-margin learning and then introduce our novel objective function.

\subsection{Soft-max Loss}
Given an input-output pair $\{\bx_i, y_i\}$, a deep neural network transforms input to a feature space representation $\bft_i$ using a function $\mathcal{F}$ parameterized by $\theta$ i.e.,  $\bft = \mathcal{F}(\bx; \theta)$. The soft-max loss can then compute the discrepancy between prediction and ground-truth in the label space as follows:
\begin{align} \label{eq:softmax}
L_{sm} = \frac{1}{N}\sum\limits_i -\log \Big( \frac{\exp(\bw^T_{y_i}\bft_i)}{\sum_j \exp( \bw^T_{j}\bft_i)} \Big),
\end{align}
where $i \in [1,N],\; j \in [1,C]$, $N$ and $C$ are number of training examples and classes respectively. It is worth noting that we have included the last fully connected layer in the definition of soft-max loss  which will be useful in further analysis. Also, for the sake of brevity, we do not mention unit biases in Eq.~\ref{eq:softmax}.

Although soft-max loss is one of the most popular choices for multi-class classification, in the following discussion, we argue that it is not suitable for class imbalanced learning due to several limitations.

%\todo{check more theoretical insights from this paper: Class Imbalance Redux} %https://ieeexplore.ieee.org/stamp/stamp.jsp?tp=&arnumber=6137280

\paragraph{Limitations of SL:}
The loss function in Eq.~\ref{eq:softmax} computes inner vector product $\langle \bw, \bft \rangle$ which measures the projection of feature representation on to each of the class vectors $\bw_j$. The goal is to perfectly align $\bft_i$ with the correct class vector $\bw_{y_i}$ such that the data likelihood is maximized. Due to the reliance of the oft-max loss on vector dot-product, it has the following limitations: 
\begin{itemize}\setlength{\itemsep}{0em}
\item No inherent mechanism to ensure max margin constraints. Computation of inter-class margin for soft-max loss is intractable \cite{elsayed2018large}. Large margin constraints promote better generalization in imbalanced distributions and robustness against input perturbations \cite{cortes1995support}. 
\item The learned projection vectors are not necessarily equi-spaced in the classification space. That is, ideally the angle between closest projection vectors should be equal (e.g., $\frac{2\pi}{k}$ in 2D where $k$ is the number of classes). However, in practice the projection vectors for majority classes occupy more angular space compared with minority classes. This has been visualized in Fig.~\ref{fig: vis} on imbalanced MNIST dataset, and leads to poor generalization to test samples.
\item The length $\parallel \textbf{w}_j \parallel_2$ of the learnt projection vectors for different classes is not necessarily the same. It has been shown in the literature that the minority class projection vectors are weaker (i.e., with less magnitude) compared with the majority classes \cite{liu2016large}. Cost-sensitive learning which artificially augments the magnitude of the minority class projection vectors has been shown to be effective for imbalance learning \cite{khan2017cost}.
\end{itemize}

\paragraph{Unsuitability of SL for Imbalanced Learning:}
We attribute the above limitations to not directly enforcing the max-margin constraints on the classification boundaries. Consider the definition of soft-max loss (Eq.~\ref{eq:softmax}) in terms of dot-products $\bw^T\bft$, we can simplify the expression as follows:
\begin{align}
L^i_{sm} &\propto  \sum_{j\neq y_i} \exp(\bw_j^T \bft_i - \bw_{y_i}^T\bft_i)
\end{align} 
The decision boundary for a class pair $\{j,k\}$ is given by the case where $\bw_j^T \mathcal{F}(\bx) = \bw_k^T \mathcal{F}(\bx)$, i.e., the class boundaries are shared between the pair of classes. Further, minimization of $L_{sm}^i$ requires $\bw_j^T \mathcal{F}(\bx) > \bw_k^T \mathcal{F}(\bx) : k \neq j$ for correct class assignment to $\bx$. This is a `relative constraint' and therefore the soft-max loss $L_{sm}$ does not necessarily: (a) reduces intra-class variations, (b) enforces a margin between each class pair.  To address these issues, we propose our new loss function next.

%\todo{check for any derivation in terms of less weight by soft-max on less frequent classes or a toy example like fig 1 in  the link in comments}
%http://home.ijasca.com/data/documents/13IJASCA-070301_Pg176-204_Classification-with-class-imbalance-problem_A-Review.pdf
%\todo{A visualization in 2D space with different 2,3,5,10 classes (or may be single case with different initializations), where soft-max learns different (unequal) sections in space for different runs.}

\begin{figure*}[!ht]
\includegraphics[trim={16cm 19cm 7cm 11cm},clip,width=1\linewidth,keepaspectratio=True]{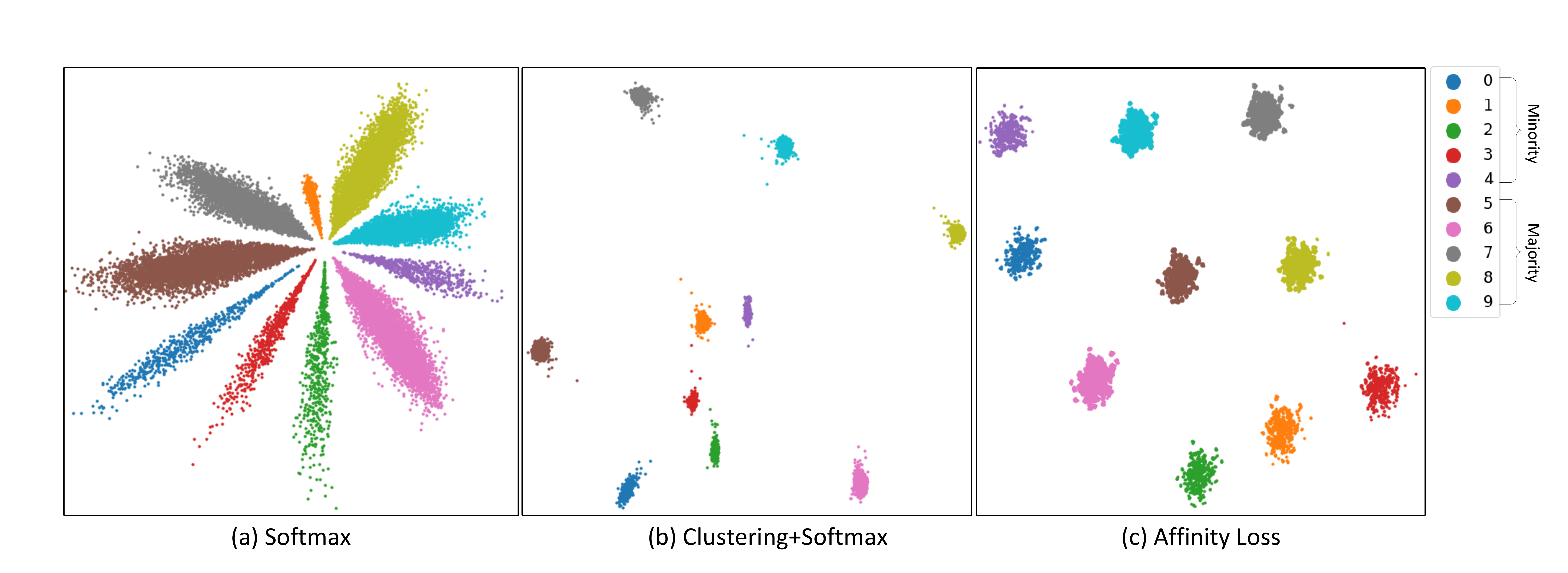}
\caption{2D feature space projections in terms of penultimate layer activations. The model is trained on imbalanced MNIST data (by retaining only 10\% of the samples for digits 0-4) using different losses: (a) soft-max loss learns floral petals in angular space, note that the minority class feature vectors are weaker (shorter in length) and occupy less angular space. (b) center loss reduces intra-class variations by performing clustering. However, the minority class vectors tend to be congested near the center and are confused amongst each other (c) the proposed affinity loss learns equi-spaced clusters of uniform shapes for both majority and minority classes.}
\label{fig: vis}
\end{figure*}

\subsection{Max-margin Learning with Hybrid Objective}

\noindent
\textbf{Euclidean space similarity measure:} Instead of computing similarities with class prototypes using vector dot-product, we propose to measure class similarities for an input feature in the Euclidean space using a Gaussian similarity measure in terms of Bergman divergence (squared $\ell^2$ distance): 
\begin{align}\label{eq:gaussian}
   d(\textbf{f}_i, \textbf{w}_j) = \exp\Big(-\frac{\parallel\textbf{f}_i-\textbf{w}_j\parallel^2}{\sigma}\Big),
\end{align}
where, $\sigma$ denotes a weighting parameter.  This provides us: \textbf{(a)} the flexibility to directly enforce margin maximizing constraints, \textbf{(b)} have equi-spaced classification boundaries for multiple classes, \textbf{(c)} control the variance of learned clusters and therefore enhancing intra-class compactness, \textbf{(d)} the freedom to use standard distance measures in Euclidean domain to measure similarity and most importantly \textbf{(e)} simultaneous classification and clustering in a single objective function.  

%\SK{Why euclidean distance based similarity measure is better than dot-product which is projection based measure. Any theoretical insights?}

\begin{theorem}
The similarity function $d(\mathbf{a}, \mathbf{b})$ is a valid similarity metric for any real-valued inputs.
\end{theorem}
\begin{proof}
The real-valued similarity function $d(\mathbf{a}, \mathbf{b})$ will define a valid similarity metric if it satisfies the following conditions \cite{li2004similarity}:
% On the similarity metric and the distance metric by Chen et al is also a very nice paper. 
\begin{itemize}\setlength{\itemsep}{0em}
    \item Non-negativity: $d(\mathbf{a}, \mathbf{b}) \geq 0$
    \item Symmetry: $d(\mathbf{a}, \mathbf{b}) = d(\mathbf{b}, \mathbf{a})$
    \item Equivalence: $d(\mathbf{a}, \mathbf{a}) = d(\mathbf{b}, \mathbf{b}) =  d(\mathbf{a}, \mathbf{b})$ iff $\mathbf{a} = \mathbf{b}$
    \item Self-similarity: $d(\mathbf{a}, \mathbf{a}) \geq d(\mathbf{a}, \mathbf{b})$ 
    \item Triangular similarity: $d(\mathbf{a}, \mathbf{b}) + d(\mathbf{b}, \mathbf{c}) \leq d(\mathbf{a}, \mathbf{c}) + d(\mathbf{b}, \mathbf{b})$
\end{itemize}
Since, all above conditions are true for $d(\cdot)$, therefore, it is a valid similarity metric. 
\end{proof}

% Figure soruce is here https://docs.google.com/spreadsheets/d/1N6MrYeUNIAwNs00k5qYmdSya81-uK06VL3BCeZz1568/edit#gid=0

\noindent
\textbf{Relation between Dot-product and Gaussian Similarity:} The proposed Gaussian similarity measure is related to the dot-product as follows:
\begin{align}\label{eq:relation}
   d(\textbf{f}_i, \textbf{w}_j) &= \exp\Big(-\frac{\parallel\bft_i\parallel^2 + \parallel\bw_j\parallel^2 - 2  \langle\textbf{w}_j, \bft\rangle}{\sigma}\Big), \\
   \langle\textbf{w}_j, \bft\rangle & = \frac{\sigma \log d(\textbf{f}_i, \textbf{w}_j)  + \parallel\bft_i\parallel^2 + \parallel\bw_j\parallel^2 }{2}
\end{align}
Intuitively, the above relation implies the dependence of soft-max loss on the scale/magnitude of feature vectors and class prototypes. It leads to two conclusios: 
(1)  It can be seen that $d(\textbf{f}_i, \textbf{w}_j)$ is bounded between $[0,1]$ since $\parallel\bft_i\parallel^2 + \parallel\bw_j\parallel^2 \geq 2  \langle\textbf{w}_j, \bft\rangle$, while $\langle\textbf{w}_j, \bft\rangle$ can have large magnitudes. (2) The Gaussian measure can be considered as an inverse chord distance when magnitudes of vectors are normalized to be equal. The dot product in that case is directly proportional to the Gaussian similarity and both similarity measures will behave similarly if no additional constraints are included in our proposed similarity measure. However, the main flexiblity with our formulation is the explicit introduction of margin constraints, which we introduce next. 

%Therefore, the Gaussian similarity is defined in terms of chord distance while the under the limit $\parallel\bft\parallel , \parallel\bw\parallel \propto \infty$, $\langle\textbf{w}_j, \bft\rangle $ grows faster compared to $d(\textbf{f}_i, \textbf{w}_j)$.

\noindent
\textbf{Enforcing margin between classes:}
 Note that some variants of soft-max loss introduce angle based margin constraints \cite{liu2017sphereface,deng2018arcface}, however, the margins in angular domain are computationally expensive and implemented only as approximations due to intractability. Our formulation allows a more direct margin penalty in the loss function. The proposed max margin loss function based on  Eq.~\ref{eq:gaussian} is given by,
\begin{align}\label{eq:maxmarg}
    L_{mm} = \sum_{j} \max \big(0, \lambda + d(\textbf{f}_i, \textbf{w}_j) - d(\textbf{f}_i, \textbf{w}_{y_i})\big) \; : j \neq y_i,
\end{align}
where $d(\textbf{f}_i, \textbf{w}_j)$ is the similarity of the sample with its true class, $d(\textbf{f}_i, \textbf{w}_{y_i})$ is its similarity with other classes, and $\lambda$ is the enforced margin.

\noindent
\textbf{Uniform classification regions:} The soft-max loss does not ensure uniform classification regions for all classes. As a result, undersampled minority classes get a shrinked representation in the feature space compared to more frequent classes. To ensure equi-distant weight vectors, we propose to apply a regularization on the  learned class weights. This regularizer is termed as a `\emph{diversity regularizer}' as it enforces all class centers ($\mathbf{w}$) to be uniformly spread out in the feature space. The diversity regularizer is formally defined as follows:
\begin{align}
R(\mathbf{w}) &= \mathbb{E}[(\parallel\mathbf{w}_j- \mathbf{w}_k\parallel^2 - \mu)^2], \; s.t.\, j < k,\\
\mu &= \frac{2}{C^2-C} \sum\limits_{j<k} \parallel\mathbf{w}_j-\mathbf{w}_k\parallel^2
\end{align}
where $\mu$ is the mean distance between all class prototypes.

\textbf{Multi-centered learning:} For challenging classification problems, the feature space may be partitioned such that all samples belonging to the same class are not co-located in a single region. Therefore, clustering all same class samples with a single prototype (class center) will not be optimal in such cases. To resolve this limitation, we introduce a novel multi-centered learning paradigm based on our max-margin framework. Instead of learning a single projection vector $\mathbf{w}_j$ for each class, the proposed framework enables learning multiple projection vectors $\{\mathbf{w}_{t}\}_j$ per-class. Specifically, we can learn $m$ projection vectors per class, where similarity of a feature vector $\textbf{f}_i$ with a class $j$ is given by:
\begin{align}
   d(\textbf{f}_i, \textbf{w}_j) =\max \big\{\exp\big(-\frac{\parallel\textbf{f}_i-\textbf{w}_{j,t}\parallel^2}{\sigma}\big)\big\}, t=[1, m] 
\end{align}
Max-margin loss is then defined similar to Eq.~\ref{eq:maxmarg} above. The overall loss function therefore becomes:
\begin{align}
    L = L_{mm} + R(\mathbf{w}).
\end{align}
The diversity regularizer for the multi-center case is enforced on the similarity between all $m*C$ prototypes.

\section{Experiments}

\begin{figure}
    \centering
    \includegraphics[width=\columnwidth,keepaspectratio=True,trim=3em 18em 3em 1em, clip=true]{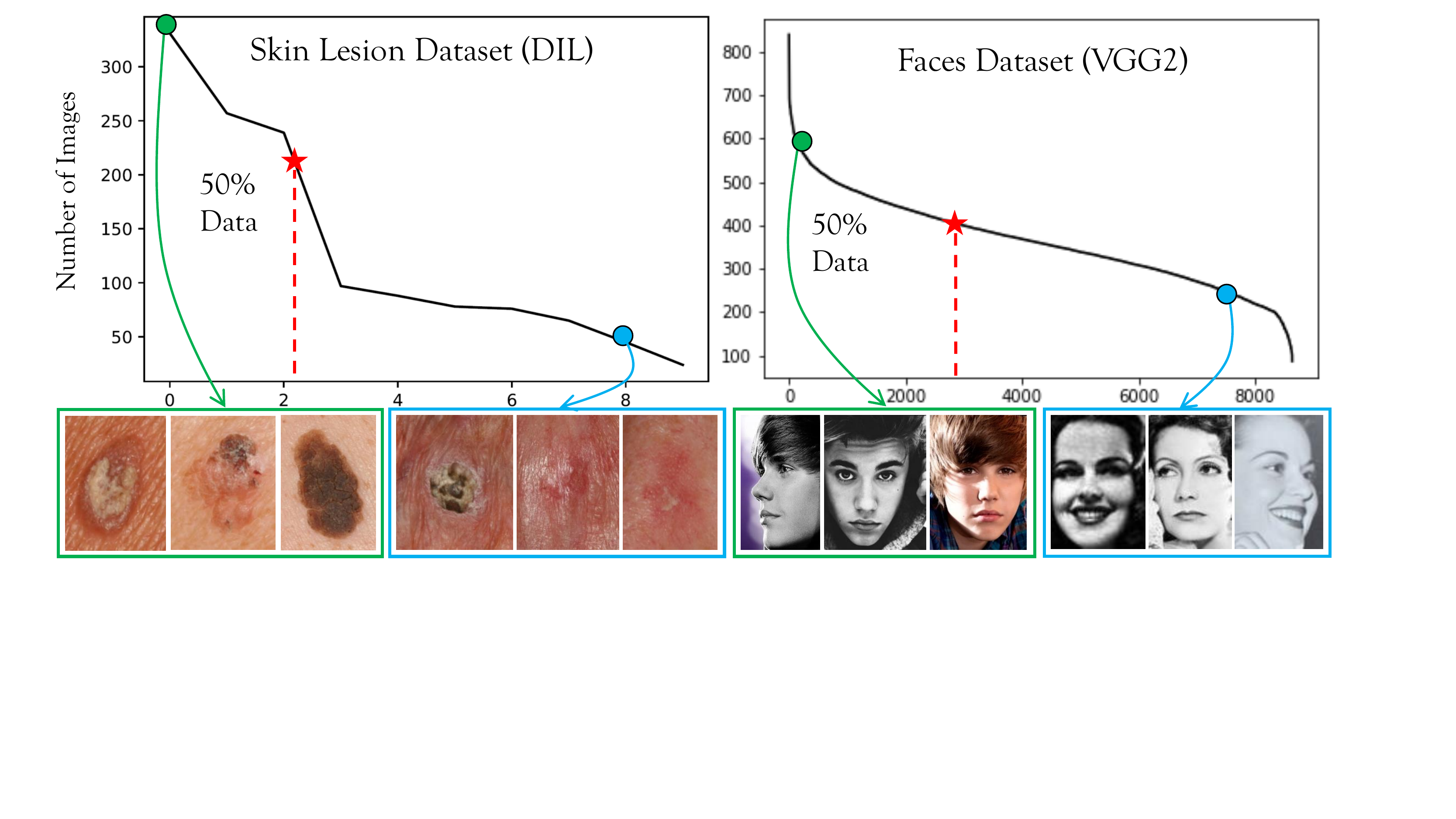}
    \caption{Data Imbalance due to long-tail distribution.}
    \label{fig:data_plot}
\end{figure}

% Figure soruce is here https://docs.google.com/spreadsheets/d/1N6MrYeUNIAwNs00k5qYmdSya81-uK06VL3BCeZz1568/edit#gid=0

To demonstrate the effectiveness of the proposed affinity loss, we perform experiments on datasets which exhibit natural imbalance. These include Dermofit Image Library (DIL) for skin lesion classification and large scale image datasets for facial verification. We further extensively evaluate various components of the proposed approach by systematically generating imbalance and introducing different levels of label noise. Through these empirical evaluations, we provide an evidence of the robustness of the proposed method against different data imbalance levels and noisy training labels. A brief description about the evaluated datasets is presented next.

\subsection{Datasets}
%Class imbalance is naturally observed in all real-world image datasets.

%We further evaluate our approach by systematically generating imbalance in datasets which have equal class representation (MNIST, CIFAR). 

 \noindent \textbf{Skin Melanoma Dataset (DIL):} 
 Edinburgh Dermofit Image Library (DIL) contains $1300$ images belonging to $10$ skin lesion categories including melanomas, seborrhoeic keratosis and basal cell carcinomas. The images are based upon diagnosis from dermatologists and dermatopathologists. The number of images vary amongst categories (between $24$ and $331$, mean $130$, median $83$), and show significant imbalance, with $50\%$ of all images belonging to only top two classes (Fig.~\ref{fig:data_plot}). Similar to \cite{ballerini2013color}, we perform two experiments, considering five and ten class splits respectively, and report results for 3-fold cross validation.

\noindent \textbf{Face Recognition: } 
Datasets used to train large scale face recognition models have natural imbalance. This is because the data is web-crawled, and images for some identities are easily available in abundance compared with others. For unconstrained face recognition, we train our model on VGG2 \cite{Cao18}, which is a large scale dataset with inherent class imbalance. We evaluate the trained network on four different datasets. These include two popular widely used benchmarks i.e., Labelled Faces in the Wild (LFW) \cite{LFWTechUpdate} and YouTube Faces (YTF) \cite{wolf2011face}. We further evaluate on Celebrities in Frontal Profile (CFP) \cite{cfp-paper} and Age Database (AgeDB) \cite{AgeDB}. %A brief description of the face datasets is given below. 

\noindent \textbf{VGG2:} facial image dataset \cite{Cao18} contains $3.31$ million images belonging to $8,631$ identities. The number of samples for each subject exhibit imbalance and vary from $80$ to $843$ with a mean of $362$. The data is collected from the Internet and has real-life variations in the form of ethnicites, head poses, illumination changes and age groups.

\noindent \textbf{LFW:} Labelled Faces in the Wild (LFW) \cite{LFWTechUpdate} contains $13,233$ static images of $5749$ individuals collected over the Internet in real-life situations. We follow the standard evaluation protocol `unrestricted with labeled outside data' \cite{LFWTechUpdate} and test on $6000$ pairs for face verification.

\noindent \textbf{YTF:} YouTube Faces (YTF) \cite{wolf2011face} has $3425$ videos belonging to $195$ different subjects. The length of video sequences varies between $48$ and $6070$ frames, with an average of $181.3$ frames. We follow the standard evaluation protocol for face verification on $5000$ video pairs.

\noindent \textbf{CFP:} contains $10$ frontal and $4$ profile view images for $500$ different identities \cite{cfp-paper}. Two evaluation protocols are used based upon the type of images in the gallery and probe: frontal-frontal (FF) and frontal-profile (FP). Each protocol has 10 runs, each with 700 face pairs (350 same and 350 different).

\noindent \textbf{AgeDB:} has $12,240$ images acquired in-the-wild for $440$ subjects \cite{AgeDB}. Along with variations across expression deformations, head poses and illumination conditions, a distinct feature of this dataset is the diversity across ages of the subjects, which ranges between $3$ and $101$ years, with an average of $49$ years. Test evaluation protocol has four groups with different age gaps (5, 10, 20 and 30 years). Each group contains ten splits, each having 600 face image pairs (300 same, 300 different). We use the most challenging split with 30 years gap.

 \noindent \textbf{Imbalanced MNIST:} Standard MNIST has $70,000$ handwritten images of digits (0-9), $60,000$ of these images are used for training (∼600/class) and the remaining 10,000 for testing (∼100/class). For this paper, we perform experiments on the standard evaluation split, as well as by systematically creating imbalance in the training set. For this, we reduce the even and odd digit samples to 10\% and 25\%. We further perform ablative study (Sec.~\ref{ablation}) by gradually introducing different imbalance ratios amongst classes and noise levels in the training labels.% and generating noisy labelled training set. 

%\noindent \textbf{Imbalanced CIFAR10:}  CIFAR-10/100 contains 60,000 images belonging to 10 and 100 classes (6000 and 600 images/class) respectively. The standard train/test split for each class is $\sim$83.3\%/16.7\% images. We evaluate our approach on the standard split as well as on artificially imbalanced splits. To imbalance the training distribution, we reduce
%the representation of even-numbered and odd-numbered classes to only 10\% and 5\% of images, respectively.

 %We evaluate our approach on the standard split as well as a artificially created imbalanced split. To imbalance the training distribution, we reduce the representation of even and odd digit classes to only 10\% and 25\% of images, respectively.

\subsection{Experimental Settings}

For experiments on DIL dataset, ResNet-18 backbone is used in combination with the proposed affinity loss.  For training the model to learn features for face verification tasks, we deploy Squeeze and Excitation (SE) networks \cite{hu2017squeeze} with ResNet-50 backbone and affinity loss. The face images are cropped and re-sized to $112\times112$ using multi-task cascaded Convolution Neural Network (CNN) \cite{7553523}. The model is trained using random horizontal flips as data augmentation.  The features extracted after the global pooling layer are then used for face verification evaluations on different datasets. The experiments on MNIST are performed on a simple network with four hidden layers having three convolution layers ($32$, $64$ and $128$ filters of $5\times5$), one fully connected layer ($128$ neurons), and an output layer. The model is trained with Stochastic Gradient Descent (SGD) optimizer with momentum and learning rate decay. For ablative study in Sec.~\ref{ablation}, we only change the output soft-max layer with the proposed Affinity loss layer and keep rest of the architecture fixed. %For experiments on CIFAR-10 dataset, we use an architecture similar to the one used for MNIST dataset. %For experiments on CIFAR-10 dataset, we use Wide-ResNet.

\subsection{Results and Analysis}

Table.~\ref{tab:DILexp} present our experimental results on DIL dataset. In Exp\#1, we report average performance for 3 fold cross validation on five classes (Actinic Keratosis, Basal Cell Carcinoma, Melanocytic Nevus, Squamous Cell Carcinoma and Seborrhoeic Keratosis). Compared with existing state of the art \cite{khan2017cost}, we achieve an absolute gain of $10.9\%$ on Exp\#1. For Exp\#2 on DIL dataset, all $10$ classes are considered. Evaluations on 3 fold cross validation in Table~\ref{tab:DILexp} show a significant performance improvement of $7.7\%$ for Exp\#2. Confusion matrix analysis for class-wise accuracy comparison in Fig~\ref{fig:conf_mat} shows that the performance boost is more pronounced for minority classes with lower representations. We attribute this to the capability of the proposed method to simultaneously optimize within class compactness by performing feature space clustering, and enhance inter-class separability by enforcing max-margin constraints. Our method achieves competitive performance on LFW and YTF datasets in Table~\ref{tab:lfw_ytf}. The performances on LFW and YTF are already saturated with many recent methods surpassing human-level results. The top performing methods on these datasets have been trained on much larger models with significantly more data and model parameters. Further evaluations on other facial recognition benchmarks achieve verification accuracies of 95.9\%, 99.5\% and 96.0\% on AgeDB30, CFP-FF and CFP-FP datasets respectively. These results prove the effectiveness of the proposed approach for large scale imbalanced learning. It is worth noting that our proposed Affinity loss does not require additional compute and memory and is easily scalable to larger datasets. This is in contrast to some of the existing loss formulations (such as triplet loss \cite{schroff2015facenet} and contrastive loss \cite{hadsell2006dimensionality}) which do enhance feature space discriminability, but suffer scalability to large data due to substantial possible combinations of training pairs, triplets or quintuplets.

\begin{table}
\centering
\scalebox{.76}{
\begin{tabular}{@{}lcc@{}}
  \toprule[0.4mm]
   \textbf{Methods} (using stand. split) &  \multicolumn{2}{c}{\textbf{Performances}}  \\
  \midrule
%  Conv DBN \cite{lee2009convolutional} & \multicolumn{2}{c}{99.2\%}  \\
%  Deep learning via embedding \cite{weston2012deep} & \multicolumn{2}{c}{98.5\%} \\
  Deeply Supervised Nets \cite{lee2015deeply} & \multicolumn{2}{c}{{99.6\%}}\\
  Generalized Pooling Func. \cite{lee2016generalizing} & \multicolumn{2}{c}{99.7\%} \\
  Maxout NIN \cite{chang2015batch} & \multicolumn{2}{c}{\textbf{99.8\%}} \\
  \midrule
  \textbf{Imbalanced} ($\downarrow$) &  CoSen CNN \cite{khan2017cost} & Affinity Loss \\
  \midrule
  Stand. split &  {99.3\%} & \textbf{99.6\%}\\ 
  %\midrule 
   10\% of  odd digits   & {98.6\%} & \textbf{99.3\%} \\
  %\hline
 % Low rep. (10\%) of  all digits &  97.2\%  & 97.6\%\\
 % \hline
   10\%  of  even digits    & {98.4\%} & \textbf{99.3\%}\\
  %\midrule 
   25\%  of  odd digits  & {98.9\%}  & \textbf{99.4\%} \\
  %\hline
%  Low rep. (25\%) of  all digits &   & \\
 % \hline
  25\% of  even digits  & {98.5\%} & \textbf{99.5\%} \\
  \bottomrule[0.4mm]
  \vspace{0.1mm}
\end{tabular}}
\caption{Evaluations on Imbalanced MNIST Database.}\vspace{-0.2cm}
\label{tab:MNISTexp}
\end{table}
%--------------------------------------------------------------

%---------------------------------------------------------------

\begin{table}
\centering
\scalebox{.85}{
\begin{tabular}{@{}lcc@{}}
  \toprule[0.4mm]
  \textbf{Methods}  &  \multicolumn{2}{c}{\textbf{Performances}}  \\
  %\cline{2-3}
  (using stand. split) & Exp\#1 (5-classes) & Exp\#2 (10-classes) \\
  \midrule
  Hierarchical-KNN \cite{ballerini2013color} &{74.3 $\pm$ 2.5\%}  & 68.8 $\pm$ 2.0\%\\
  Hierarchical-Bayes \cite{ballerini2012non} & {69.6 $\pm$ 0.4\%} & 63.1 $\pm$ 0.6\% \\
  Flat-KNN \cite{ballerini2013color} & {69.8 $\pm$ 1.6\%} & 64.0 $\pm$ 1.3\%\\
  \midrule
 
  CoSen CNN \cite{khan2017cost}  & \sbest{80.2 $\pm$ 2.5\% }& \sbest{72.6 $\pm$ 1.6\%}  \\
 Affinity Loss & \best{91.1 $\pm$ 1.7\%} & \best{80.3  $\pm$ 2.1\%} \\ 
  \bottomrule[0.4mm]
  \vspace{0.1mm}
\end{tabular}
}
\caption{Evaluation on DIL Database.}\vspace{-0.2cm}
\label{tab:DILexp}
\end{table}
%---------------------------------------------------------------

\begin{figure}
    \centering
    \begin{tabular}{cc}
    \includegraphics[width=.5\linewidth,keepaspectratio=True,trim=0em 0em 4em 2em, clip=true]{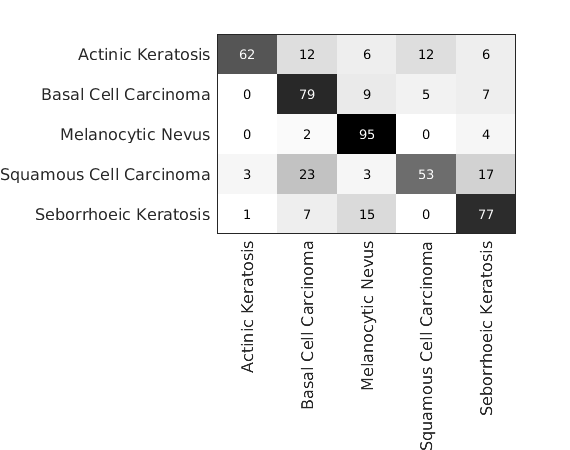}&
    \includegraphics[width=.5\linewidth,keepaspectratio=True,trim=0em 0em 4em 2em, clip=true]{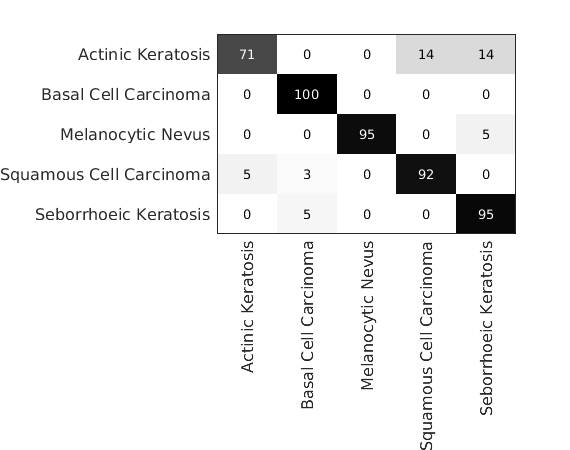}  \\
    {\footnotesize (a) CosSen CNN \cite{khan2017cost}} & {\footnotesize (b) Affinity Loss }
    \end{tabular}
    \caption{Confusion matrices for Exp\#1 on DIL dataset. }
    \label{fig:conf_mat}
\end{figure}

% \usepackage{multirow}
\begin{table}[]
\centering
 \resizebox{.99\columnwidth}{!}{
\begin{tabular}{l c c c c}
\toprule[0.4mm]
\textbf{Methods}  &  \#\textbf{Models}  &  \textbf{Train Data}  & \textbf{LFW}             & \textbf{YTF} \\
  \midrule
DeepFace \cite{taigman2014deepface}   &  3  &  4M  & 97.35 & 91.4           \\
FaceNet \cite{schroff2015facenet}     &  1  & 200M & 99.63 & 95.4           \\
Web-scale \cite{taigman2015web}       &  4  & 4.5M & 98.37 & -              \\
VGG Face \cite{parkhi2015deep}        &  1  & 2.6M & 98.95 & 97.3           \\
DeepID2+ \cite{sun2015deeply}         &  25 & 0.3M & 99.47 & 93.2           \\
Baidu  \cite{liu2015targeting}        &  1  & 1.3M & 99.13 & -              \\
Center Face \cite{wen2016discriminative} & 1 & 0.7M & 99.28 & 94.9           \\
Marginal Loss \cite{deng2017marginal} &  1  & 4M & 99.48 & 95.98          \\
Noisy Softmax \cite{chen2017noisy}  &  1  & Ext. WebFace & 99.18  & 94.88          \\
%\midrule 
Range Loss  \cite{zhang2017range}   &  1  &  1.5M & 99.52 & 93.7           \\
Augmentation \cite{masi2016we}      &  1  & WebFace & 98.06  & -              \\
Center Invariant Loss   \cite{wu2017deep}  &  1  & WebFace & 99.12 & 93.88   \\
Feature transfer  \cite{yin2018feature}    &  1  & 4.8M    & 99.37 & -   \\
%\midrule
Softmax+Contrastive \cite{sun2014deep} &  1   & WebFace   & 98.78  & 93.5           \\
Triplet Loss  \cite{schroff2015facenet} & 1   & WebFace  & 98.7  & 93.4           \\
Large Margin Softmax  \cite{liu2016large} & 1 & WebFace  & 99.10 & 94.0 \\
Center Loss  \cite{wen2016discriminative}   &  1  & WebFace & 99.05  & 94.4   \\
SphereFace  \cite{liu2017sphereface}        &  1  & WebFace & 99.42  & 95.0  \\
CosFace  \cite{Wang_2018_CVPR}   & 1  & WebFace  & 99.33    & 96.1           \\
LMLE  \cite{huang2016learning} &   1   & WebFace  & {99.51}    & 95.8          \\
%CLMLE \cite{} & 1 & WebFace & 99.62 & 96.5 \\
%\midrule
Affinity Loss   &  1  & VGG2 & \textbf{99.65}    &  \textbf{97.3}              \\
\bottomrule[0.4mm]
\vspace{0.1mm}
\end{tabular}}
\caption{Face Verification Performance on LFW and YTF datasets.}\vspace{-0.2cm}
\label{tab:lfw_ytf}
\end{table}

% \begin{table}[]
% \centering
% \scalebox{.95}{
% \begin{tabular}{llll}
%  \toprule[0.4mm]
%  Dataset & Performance &Dataset & Performance \\
%  \midrule
% AgeDB  & 95.9  & CPLFW  &       \\
% CFP-FF & 99.63 & CFP-FP & 94.01 \\
% \bottomrule

% \end{tabular}
% }
% \caption{Evaluation Face datasets.}\vspace{-0.2cm}
% \label{tab:face}
% \end{table}

\begin{figure}
    \centering
    \includegraphics[width=\columnwidth, trim=0.5em 0.5em 0.5em 0.5em, clip=true]{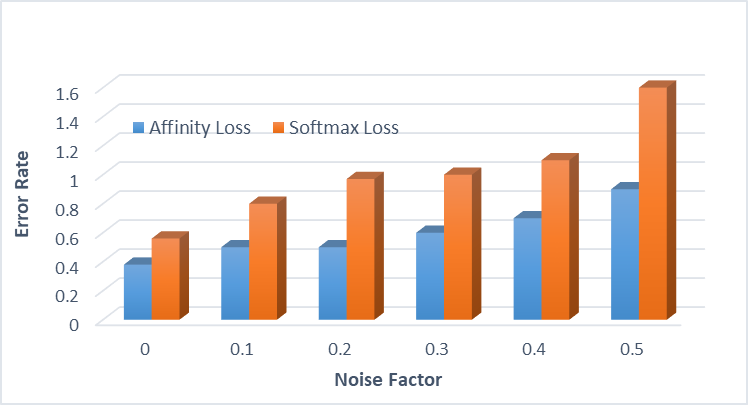}
    \caption{The effect of label noise on the soft-max and affinity loss functions. }
    \label{fig:noise}
\end{figure}

\subsection{Generalization} %\todo{Show performance by reducing the training set; or by gradually increasing the class imbalance ratio. See 4.2.2 of \cite{elsayed2018large} }
To test the generalization of the proposed method for different imbalance levels, we gradually reduce the training set by changing the representation of the minority class samples on MNIST data. Specifically, we gradually alter the majority to minority class ratios (up-to $1:0.025$) by randomly dropping samples of the first five digits ($0-4$). Under these settings, we therefore have significantly lower representation for half of the classes. The experimental results in terms of error rates against fraction of retained minority class samples are shown in Fig.~\ref{fig:ratios_minority}. We also repeat these experiments for standard soft-max loss. The comparison in Fig.~\ref{fig:ratios_minority} demonstrates a consistently superior performance of the proposed loss function across all settings. The effect on achieved performance is more noticeable for larger imbalance levels between majority and minority classes. The proposed Affinity loss enhances inter-class separability irrespective of the class frequencies by enforcing margin maximization constraints. Soft-max loss does not have inherent margin learning characteristics. Further, compared with soft-max loss, where intra-class variations can vary across classes depending upon their representative samples, affinity loss learns uniformed sized clusters.
%in the feature space for all classes, which results in improved generalization. 
As visualized in Fig.~\ref{fig: vis}, feature space within class disparities are observed for soft-max loss with minority classes occupying compact regions compared with their majority counterparts. In comparison, our proposed loss formulation is flexible, and allows learnt class prototypes to be equi-spaced and form uniformly shaped clusters. This reduces bias towards the less frequent cases and enhances the overall generalization capabilities, thus yielding a more discriminatively learnt feature space and an improved performance.%, With our proposed affinity loss, we enforce each class to occupy uniform sized regions in the feature space, which reduces the bias towards less frequent classes.
%t0-be-revised-This helps reduce intra-class variations.   Our proposed formulation is quite flexible and allows enforcing constraints on the learnt margins by performing feature space clustering with uniform sized cluster centers.The max-margin constraints further ensure the classes are optimally separated from each other irrespective of their number of samples.

\begin{figure}
\centering
\includegraphics[width=0.8\linewidth,keepaspectratio=True,trim=1em 1.5em 1em 1em, clip=true]{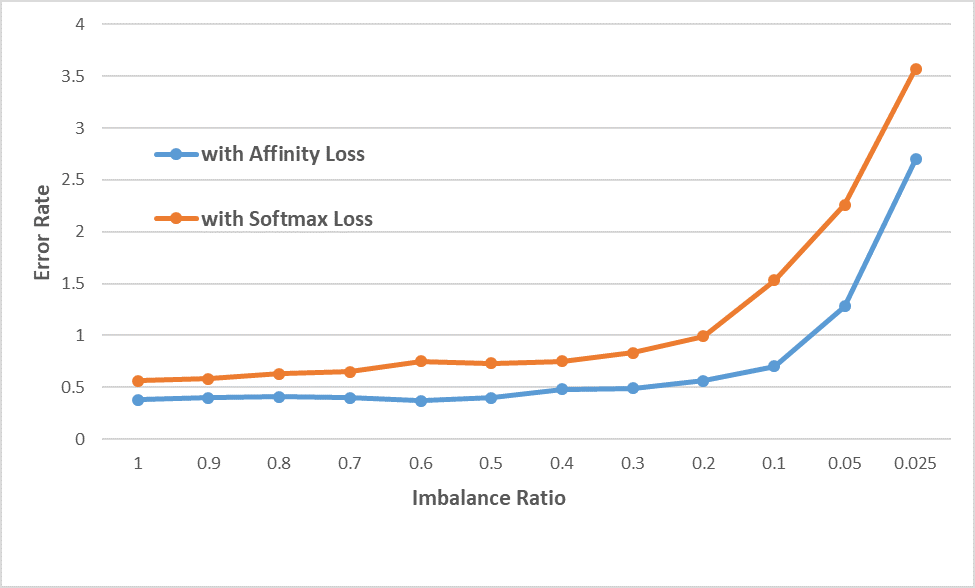}
\caption{Robustness analysis against different imbalance levels (fraction of retained minority class samples)}
\label{fig:ratios_minority}
\end{figure}

\subsection{Robustness against Noisy Labels} \label{Robustness against Noisy Labels} For many real-world applications, the acquired data has noisy labels, and generalization of the learning methods against label noise is highly desirable \cite{sheng2008get,vinyals2016matching,huang2018discriminative}. To check the robustness of our proposed approach against noisy labels in the training data, we randomly flip the classes of MNIST training samples. The fraction of the miss-labelled samples is gradually increased from $10\%$ to $50\%$ with an increment of $10\%$. In order to avoid over-fitting on the noisy data, we deploy early stopping \cite{yao2007early}, and finish training when the performance on a held-out cross validation set starts to degrade. For comparison, we repeat all experiments using standard soft-max loss. The experimental results in Fig.~\ref{fig:noise} show that that the proposed Affinity loss performs better across the entire range of different noise levels. Although, the performance for both soft-max and affinity losses degrades with increasing noise factors, the proposed affinity loss shows more robustness, specially for larger noise ratios, with comparatively less performance degradation. The multi-centered learning in our loss provides flexibility to the noisy samples to associate themselves with class prototypes which are different from the non-noisy and clean samples. %This is useful in learning from noisy data. 

% \subsection{Confusion Matrix Analysis} \todo{Show that the performance for minority classes improves, compared with baseline}

% \begin{figure}
%     \centering
%     \begin{subfigure}  
%     \centering
%     \includegraphics[width=1\linewidth,keepaspectratio=True]{figs/tnnls.png}
%     \caption{CosSen CNN \cite{khan2017cost}}
%     \end{subfigure}
%     \begin{subfigure} 
%     \centering
%     \includegraphics[width=1\linewidth,keepaspectratio=True]{figs/ours.png}  
%     \caption{Affinity Loss}
%     \end{subfigure}
%     \caption{Confusion matrices for Exp-1 on DIL dataset. }
%     \label{fig:intro_fig}
% \end{figure}

% \begin{figure}%
%     \centering
%     \subfloat[label 1]{{\includegraphics[width=.49\linewidth,keepaspectratio=True]{figs/tnnls.png} }}%
 
%     \subfloat[label 1]{{\includegraphics[width=.49\linewidth,keepaspectratio=True]{figs/ours.png} }}%
%     \caption{2 Figures side by side}%
%     \label{fig:example}%
% \end{figure}

\begin{SCfigure}
    \centering
    \includegraphics[width=.5\linewidth, trim=1em 0.5em 1em 0.5em, clip=true]{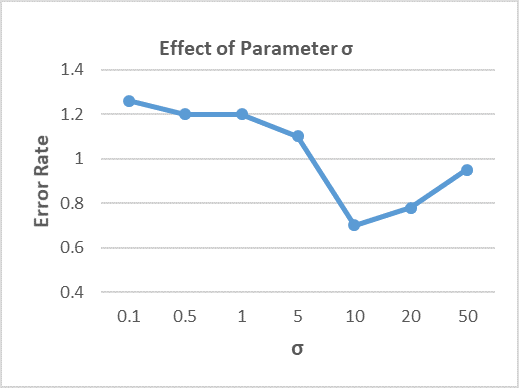}
    \caption{Effect of changing parameter $\sigma$ that controls the spread of clusters. Results on imbalanced MNIST show that increasing the cluster variance above a certain point results in overlapped clusters and higher error rate. }
    \label{fig:sigma}
\end{SCfigure}

%*  Cosine similarity and dot products work better for high -dimensional spaces while Euclidean metrics may suffer from curse of dimensionality, any comments on that ...

\subsection{Ablation}
\label{ablation}
% \textbf{Distance Metrics:}
% * Does other metrics induced by norms e.g., Minkowski distance, Manhattan distance also work better. Also, fractional distance metrics are proposed to avoid curse of dimensionality in "On the surprising bahevioru of distance metrics" (cited 1100+ times). Can it be useful here? 
% https://numerics.mathdotnet.com/distance.html

\noindent \textbf{Number of Cluster Centers:} A unique aspect of the proposed affinity loss is its multi-centered learning which provides us the flexibility to have multiple class prototypes for each class. Here, we perform experiments on the imbalanced MNIST dataset ($10\%$ representation for first five digits), by gradually changing the number of representative prototypes $m$ per class from $1$ to $20$. The experimental results in terms of error rates vs prototypes $m$ in Fig.~\ref{fig:clusters_differet} show that the best performance is achieved for $m \geq 5$. Fewer prototypes per class ($m \leq 5$) yield relatively poor performance. The proposed method performs consistently when prototypes are increased beyond $5$. Such multi-centered learning supports diversity in input samples. It is specifically helpful in scenarios with complex data distributions where large differences are observed amongst samples of the same class. Such diverse samples might not necessarily cluster around a single region, and could form multiple clusters by virtue of the proposed multi-centered learning mechanism. Furthermore, our experiments in Sec.~\ref{Robustness against Noisy Labels} show that by providing flexible class prototypes, multi-centered learning proves an effective and robust scheme against noisy samples.

\begin{figure}
    \centering
    \includegraphics[keepaspectratio=Ture,width=\linewidth,trim=1em 0.5em 1em 0.5em, clip=true]{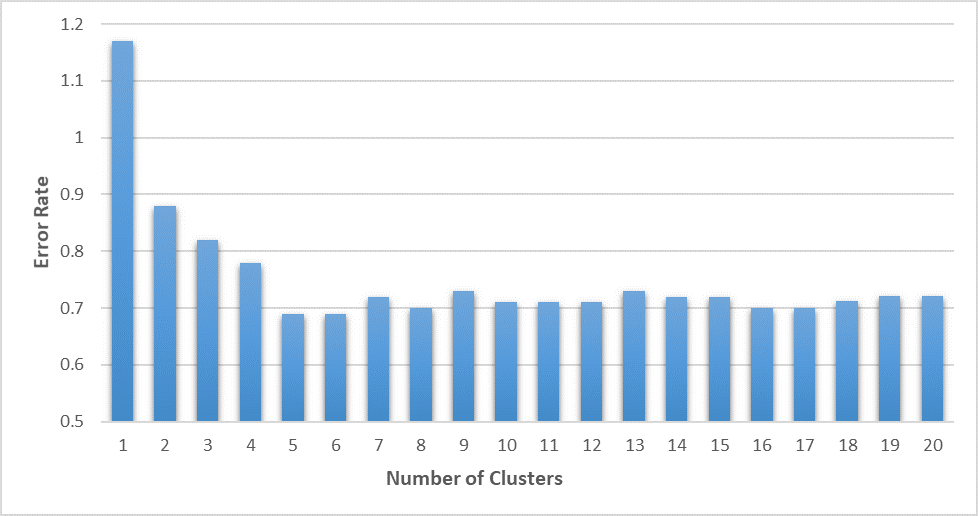}
    \caption{Performance for different number of clusters per class.}
    \label{fig:clusters_differet}
\end{figure}

%These experiments validate that the multi-centered learning provides improved performance. Further, it proves effective in learning from data with noisy labels. %These experiments validate the benefits of the multi-centered learning scheme adapted. 

\noindent \textbf{Cluster Spread $\sigma$:} The parameter $\sigma$ in Eq.~\ref{eq:gaussian} determines the cluster spread and helps achieve uniform intra-class variations. Our 2D visualization of the learnt features in Fig.~\ref{fig: vis} demonstrate that the clusters for each class are uniformly sized for both the majority and minority classes. This is in contrast to the traditional soft-max loss, where shrinked feature space regions are observed for the minority classes. For our proposed loss formulation, the size of the cluster is directly related with the value of parameter $\sigma$, with larger $\sigma$ indicating larger variance for a cluster. We perform experiment on imbalanced MNIST dataset for different values of the the parameter $\sigma=\{.1,.5,1,5,1e1,2e1,5e1\}$. The results in Fig.~\ref{fig:sigma} show that the optimal performance is achieved for values of $\sigma$  between $5$ and $20$. Very high values of $\sigma$ results in larger cluster spreads causing overlaps and confusion amongst classes and lower classification performance. %Lower values of $\sigma$ also yield poor performance, since this requires high variance in the activations of the preceding layer

\begin{table}[]
\centering
\scalebox{.95}{
\begin{tabular}{llc}
 \toprule[0.4mm]
 Distance $d$ & Similarity &Performance  \\
 \midrule
$||a-b||^1$ & $\exp{-(d/\sigma})$ & $99.3$\\
$||a-b||^2$ & $\exp{-(d/\sigma})$ & $99.3$ \\
$||a-b||^1$ & $1/(1+d)$ & $-$\\
$||a-b||^2$ & $1/(1+d)$ & $99.1$\\

\bottomrule

\end{tabular}}
\vspace{0.5em}
\caption{Evaluation with different combinations of distance and similarity measures.} 
\label{tab:comb_sim_dis}
\end{table}

\noindent \textbf{Distance and Similarity Metrics:} Our original affinity loss formulation in Eq.~\ref{eq:gaussian} first computes the squared $\ell_2$ distance between the feature $\textbf{f}$ and class prototype $\textbf{w}$, which is then converted to a similarity measure using the Gaussian metric. In this experiment, we evaluate different combinations of distance and similarity metrics. $\ell_1$ and $\ell_2$ metrics are used to compute distances, whereas Gaussian and inverse distance (defined by $\frac{1}{1+x}$) are the two similarity measures. We perform these experiments on imbalanced MNIST data (by retaining 10\% samples for first five digits). Table.~\ref{tab:comb_sim_dis} shows our evaluation results. The proposed scheme works well with all combinations except for $\ell_1$ distance and Gaussian similarity, where it fails to converge. The best performance is achieved for Gaussian similarity in combination with squared $\ell_2$ distance.

\section{Conclusion}

Class imbalance is ubiquitous in natural data  and learning from such data is an unresolved challenge. The paper proposed a flexible loss formulation, aimed at producing a generalizable large margin classifier, to tackle class imbalance learning using deep networks. Based upon Euclid space affinity defined using Gaussian similarity on Bregmen divergence, the proposed loss jointly performs feature space clustering and max-margin classification. It enables learning uniform sized equi-spaced clusters in the feature space, thus enhancing between class separability and reducing intra-class variations. The proposed scheme complements existing regularizer such as weight decays, and can be incorporated with different architectural backbones without incurring additional compute overhead. Experimental evaluations validate the effectiveness of the affinity loss for face verification and image classification benchmarks involoving imbalanced data. 
% \subsection{Discussion} 

{\small
\bibliographystyle{ieee}
\bibliography{egbib}
}

\end{document}